\documentclass[11pt]{article}    
\usepackage[letterpaper,margin=1in]{geometry}
\usepackage{xcolor}
\usepackage{authblk}
\usepackage[utf8]{inputenc} 
\usepackage{graphicx}
\usepackage{url}
\usepackage{booktabs}
\usepackage{amsmath,amsfonts}
\usepackage{enumerate}
\usepackage{multirow}
\usepackage{color}
\usepackage{ttstyle}
\usepackage{url}
\usepackage{hyperref}
\usepackage{comment}
\usepackage{stmaryrd}
\usepackage{bm}

\usepackage{array}
\newcolumntype{C}[1]{>{\centering\arraybackslash\hspace{0pt}}p{#1}}

\usepackage{pifont}

\begin{document}

\title{Interactive Learning from Multiple Noisy Labels}
\author[1]{Shankar Vembu}
\author[2]{Sandra Zilles}
\affil[1]{Donnelly Center for Cellular and Biomolecular Research,\break University of Toronto, Toronto, ON, Canada}
\affil[2]{Department of Computer Science,\break University of Regina, Regina, SK, Canada}
\date{}

\maketitle

\begin{abstract}
Interactive learning is a process in which a machine learning algorithm is provided with meaningful, well-chosen examples as opposed to randomly chosen examples typical in standard supervised learning. In this paper, we propose a new method for interactive learning from multiple noisy labels where we exploit the disagreement among annotators to quantify the easiness (or meaningfulness) of an example. We demonstrate the usefulness of this method in estimating the parameters of a latent variable classification model, and conduct experimental analyses on a range of synthetic and benchmark datasets. Furthermore, we theoretically analyze the performance of perceptron in this interactive learning framework.
\end{abstract}

\section{Introduction}
We consider binary classification problems in the presence of a teacher, who acts as an intermediary to provide a learning algorithm with meaningful, well-chosen examples. This setting is also known as curriculum learning \cite{BengioLCW09,KhanZM11,Zhu13} or self-paced learning \cite{KumarPK10,Jiang14,JiangMZSH15} in the literature. Existing practical methods \cite{KumarPK10,LeeG11} that employ such a teacher operate by providing the learning algorithm with easy examples first and then progressively moving on to more difficult examples. Such a strategy is known to improve the generalization ability of the learning algorithm and/or alleviate local minima problems while optimizing non-convex objective functions.

In this work, we propose a new method to quantify the notion of easiness of a training example. Specifically, we consider the setting where examples are labeled by multiple (noisy) annotators \cite{SnowOJN08,DekelS09,RaykarYZVFBM10,YanRFRD14}. We use the disagreement among these annotators to determine how easy or difficult the example is. If a majority of annotators provide the same label for an example, then it is reasonable to assume that the training example is easy to classify and that these examples are likely to be located far away from the decision boundary (separating hyperplane). If, on the other hand, there is a strong disagreement among annotators in labeling an example, then we can assume that the example is difficult to classify, meaning it is located near the decision boundary. In the paper by Urner \emph{et al.} \cite{UrnerBS12}, a strong annotator always labels an example according to the true class probability distribution, whereas a weak annotator is likely to err on an example whose  neighborhood is comprised of examples from both classes, i.e., whose neighborhood is label heterogeneous. In other words, both strong and weak annotators do not err on examples far away from the decision boundary, but weak annotators are likely to provide incorrect labels near the decision boundary where the neighborhood of an example is heterogeneous in terms of its labels. There are a few other theoretical studies where weak annotators were assumed to err in label heterogeneous regions \cite{JabbariHZ12,MalagoCR14}. The notion of annotator disagreement also shows up in the multiple teacher selective sampling algorithm of Dekel \emph{et al.} \cite{DekelGS12}. This line of research indicates the potential of using annotator disagreement to quantify the easiness of a training example. 

To the best of our knowledge, there has not been any work in the literature that investigates the use of annotator disagreement in designing an interactive learning algorithm. We note that a recent paper \cite{SharamanskaHHQ16} used annotator disagreement in a different setting, namely as privileged information in the design of classification algorithms. Self-paced learning methods \cite{KumarPK10,Jiang14,JiangMZSH15} aim at simultaneously estimating the parameters of a (linear) classifier and a parameter for each training example that quantifies its easiness. This results in a non-convex optimization problem that is solved using alternating minimization. Our setting is different as the training example is comprised of not just a single (binary) label but multiple noisy labels provided by a set of annotators, and we use the disagreement among these annotators (which is fixed) to determine how easy or difficult a training example is. We note that it is possible to parameterize the easiness of an example as described in Kumar \emph{et al.}'s paper \cite{KumarPK10} in our framework and use it in conjunction with the disagreement among annotators.

Learning from multiple noisy labels \cite{SnowOJN08,DekelS09,RaykarYZVFBM10,YanRFRD14} has been gaining traction in recent years due to the availability of inexpensive annotators from crowdsourcing websites like Amazon's \emph{Mechanical Turk}. These methods typically aim at learning a classifier from multiple noisy labels and in the process also estimate the annotators' expertise levels. We use one such method \cite{RaykarYZVFBM10} as a test bed to demonstrate the usefulness of our interactive learning framework. 

\subsection{Problem Definition and Notation} Let $\Xcal \subseteq \R^n$ denote the input space. The input to the learning algorithm is a set of $m$ examples with corresponding (noisy) labels from $L$ annotators denoted by $S=\left\{\left(x_i,y_i^{(1)},y_i^{(2)},\ldots,y_i^{(L)}\right)\right\}_{i=1}^m$ where 
$\left(x_i,y_i^{(\ell)}\right) \in \Xcal \times \{\pm 1\}$, for all $i \in \{1,\ldots,m\}$ and $\ell \in \{1,\ldots,L\}$. Let $z_1, z_2, \ldots, z_L \in [0,1]$ denote the annotators' expertise scores, which is not known to the learning algorithm. A strong annotator will have a score close to one and a weak annotator close to zero. The goal is to learn a classifier $f: \Xcal \to\{\pm 1\}$ parameterized by a weight vector $w \in \R^{n}$, and also estimate the annotators' expertise  scores $\{z_1, z_2, \ldots, z_L\}$. In this work, we consider linear models $f(x)=\ip{w}{x}$, where $\ip{\cdot}{\cdot}$ denotes the dot-product of input vectors.

\section{Learning from Multiple Noisy Labels}
\label{sec:crowdsourcing_method}
One of the algorithmic advantages of interactive learning is that it can potentially alleviate local minima problems in latent variable models \cite{KumarPK10} and also improve the generalization ability of the learning algorithm. A latent variable model that is relevant to our setting of learning from multiple noisy labels is the one proposed by Raykar \emph{et al.} \cite{RaykarYZVFBM10} to learn from crowdsourced labels. For squared loss function,\footnote{We consider squared loss function to describe our method and in our experiments for the sake of convenience. The method can be naturally extended to the classification model described in Raykar \emph{et al.}'s paper \cite{RaykarYZVFBM10}. Also, we note that it is perfectly valid to minimize squared loss function for classification problems \cite{RifkinYP03}.} i.e., regression problems and a linear model,\footnote{Although we consider linear models in our exposition, we note that our method can be adapted to accommodate any classification algorithm that can be trained with weighted examples.} the weight vector $w$ and the annotators' expertise scores (the latent variable) $\{z_{\ell}\}$ can be simultaneously estimated using the following iterative updates:
\begin{equation}
\label{eq:opt_crowds}
\begin{aligned}
& \hat{w} = \argmin_{w \in \R^n} \frac{1}{m} \sum\limits_{i=1}^m \left( \ip{w}{x_i} - \hat{y}_i \right)^2 + \lambda \|w\|^2 \  , \quad \text{with } \hat{y}_i = \frac{\sum_{\ell=1}^L \hat{z}_\ell y_i^{(\ell)}}{\sum_{\ell=1}^L \hat{z}_\ell} \ ;\\
& \frac{1}{\hat{z}_\ell} = \frac{1}{m} \sum_{i=1}^m \left(y_i^{(\ell)} - \ip{\hat{w}}{x_i} \right)^2 \ , \quad \text{for all } \ell \in \{1,\ldots,L\}\ ,
\end{aligned}
\end{equation}
where $\lambda$ is the regularization parameter. Intuitively, the updates estimate the score $z$ of an annotator based on her performance (measured in terms of squared error) with the current model $\hat{w}$, and the label of an example is adjusted $\{\hat{y}_i\}$ by taking the weighted average of all its noisy labels from the annotators. In practice, the labels $\{\hat{y}_i\}$ are initialized by taking a majority vote of the noisy labels. The above updates are  guaranteed to converge only to a locally optimum solution.

We now use the disagreement among annotators in the regularized risk minimization framework. For each example $x_i$, we compute the disagreement $d_i$ among annotators as follows: 
\begin{equation}
\label{eq:disagree}
d_i = \sum_{\ell=1}^L \sum_{\ell^\prime=1}^L  \left(y_i^{(\ell)} - y_i^{(\ell^\prime)}\right)^2 \ ,
\end{equation}
and solve a weighted least-squares regression problem:
\begin{equation}
\label{eq:opt_interactive}
\begin{aligned}
\hat{w} = \argmin_{w \in \R^n} \frac{1}{m} \sum\limits_{i=1}^m g(d_i) \left( \ip{w}{x_i} - \hat{y}_i \right)^2 + \lambda \|w\|^2 \ ,
\end{aligned}
\end{equation}
where $g: \R \to [0,1]$ is a monotonically decreasing function of the disagreement among annotators, and iteratively update $\{z_{\ell}\}$ using:
\begin{equation}
\label{eq:opt_interactive2}
\begin{aligned}
\frac{1}{\hat{z}_\ell} = \frac{1}{m} \sum_{i=1}^m g(d_i) \left(y_i^{(\ell)} - \ip{\hat{w}}{x_i} \right)^2 \ , \quad \text{for all }\ell \in \{1,\ldots,L\}\ .
\end{aligned}
\end{equation}

In our experiments, we use $g(d) = (1 + e^{\alpha d})^{-1}$. The parameter $\alpha$ controls the reweighting of examples. Large values of $\alpha$ place a lot of weight on examples with low disagreement among labels, and small values of $\alpha$ reweight all the examples (almost) uniformly as shown in Figure \ref{fig:reweight_examples}. The parameter $\alpha$ is a hyperparameter that the user has to tune akin to tuning the regularization parameter.
\begin{figure}[!t]
  \centering
    \includegraphics[width=0.9\textwidth]{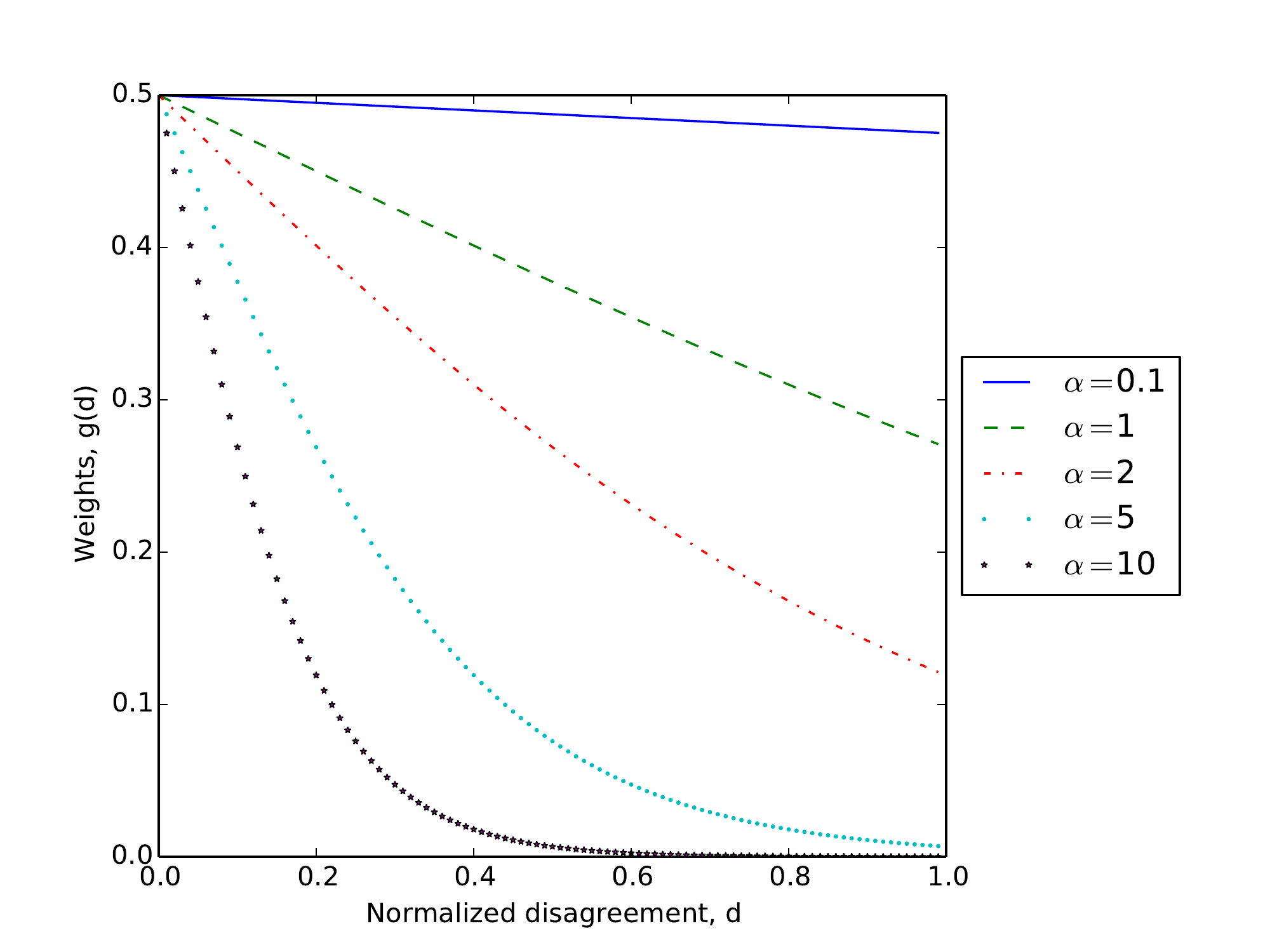}
    \caption[ ]{Example reweighting function.}
    \label{fig:reweight_examples}
\end{figure}

The optimization problem \eq{eq:opt_interactive} has a closed-form solution. Let $X \in \R^{m \times n}$ denote the matrix of inputs, $D \in \R^{m \times m}$ denote a diagonal matrix whose diagonal entries are $g(d_i)$, for all $i \in \{1,\ldots,m\}$ and $\hat{y}$ denote the (column) vector of labels. The solution is given by: 
$\hat{w}=(X^\top D X + \lambda \iden)^{-1} X^\top D \hat{y}$, where $\iden$ is the identity matrix. Hence, optimization solvers used to estimate the parameters in regularized least-squares regression can be adapted to solve this problem by a simple rescaling of inputs via $X \leftarrow \sqrt{D}X$ and $\hat{y} \leftarrow \sqrt{D}\hat{y}$.

In the above description of the algorithm, we fixed the weights $g(\cdot)$ on the examples. Ideally, we would want to reweight the examples uniformly as learning progresses. This can be done in the following way. Let $\Pcal_{\Xcal}$ denote some probability distribution induced on the examples via $g(\cdot)$. In every iteration $t$ of the learning algorithm, we pick one of $\Pcal_{\Xcal}$ or the uniform distribution based on a Bernoulli trial with success probability $1/t^c$ for some fixed positive integer $c$ to ensure that the distribution on examples converges to a uniform distribution as learning progresses. Unfortunately, we did not find this to work well in practice and the parameters of the optimization problem did not converge as smoothly as when fixed weights $g(\cdot)$ were used throughout the learning process. We leave this as an open question and use fixed weights in our experiments.

\section{Mistake Bound Analysis}
In this section, we analyze the mistake bound of perceptron operating in the interactive learning framework. The algorithm is similar to the classical perceptron, but the training examples are sorted based on their distances from the separating hyperplane and fed to the perceptron starting from the farthest example. 
The theoretical analysis requires estimates of margins of all examples. We describe a method to estimate the margin of an example and also its ground-truth label (from the multiple noisy labels) in the Appendix. We would like to remark that the margin of examples is needed only to prove the mistake bound. In practice, the perceptron algorithm can directly use the disagreement among annotators \eq{eq:disagree}.

\begin{theorem}[Perceptron \cite{Novikoff62}]
Let $((x_1,y_1), \dots, (x_T,y_T))$ be a sequence of training examples with $\|x_t\| \leq R$ for all $t \in \{1,\ldots,T\}$. Suppose there exists a vector $u$ such that $y_t \ip{u}{x_t} \geq \gamma$ for all examples. Then, the number of mistakes made by the perceptron algorithm on this sequence is at most $(R /\gamma)^2 \|u\|^2$.
\end{theorem}
The above result is the well-known mistake bound of perceptron and the proof is standard. We now state the main theorem of this paper.
\begin{theorem}
\label{main_theorem}
Let $((x_1,\hat{y}_1,\hat{\gamma}_1), \dots, (x_T,\hat{y}_T,\hat{\gamma}_T))$ be a sequence of training examples along with their label and margin estimates, sorted in descending order based on the margin estimates, and with $\|x_t\| \leq R$ for all $t \in \{1,\ldots,T\}$. Let $\hat{\gamma}=\min (\hat{\gamma}_1,\dots,\hat{\gamma}_T)=\hat{\gamma}_T$ and $~K=\lceil R/\hat{\gamma} \rceil-1$. Suppose there exists a vector $u$ such that $\hat{y}_t \ip{u}{x_t} \geq \hat{\gamma}$ for all examples. Divide the input space into $K$ equal regions, so that for any example $x_{t_k}$ in a region $k$ it holds that $\hat{y}_{t_k}\ip{x_{t_k}}{u} \geq k\hat{\gamma}$. Let $\{\varepsilon_1,\ldots,\varepsilon_K\}$ denote the number of mistakes made by the perceptron in each of the $K$ regions, and let $\varepsilon=\sum_k \varepsilon_k$ denote the total number of mistakes. Define $\varepsilon_s=\sqrt{1/K\sum_k (\varepsilon_k-\varepsilon/K)^2}$ to be the standard deviation of $\{\varepsilon_1,\ldots,\varepsilon_K\}$.

Then, the number of mistakes $\varepsilon$ made by the perceptron on the sequence of training examples is bounded from above via:
\begin{align*}
\sqrt{\varepsilon} \leq \frac{R\|u\| + \sqrt{R^2\|u\|^2+\varepsilon_s K(K+1)^2\sqrt{K-1}\hat{\gamma}^2}}{\hat{\gamma}(K+1)} \ . 
\end{align*}
\end{theorem}

We will use the following inequality in proving the above result.
\begin{lemma}[Laguerre-Samuelson Inequality \cite{Samuelson68}]
Let $(r_1, \dots, r_n)$ be a sequence of real numbers. Let $\bar{r} = \sum_i r_i /n$ and $s=\sqrt{1/n\sum_i (r_i-\bar{r})^2}$ denote their mean and standard deviation, respectively. Then, the following inequality holds for all $i \in \{1,\ldots,n\}$: $\bar{r} -s\sqrt{n-1} \leq r_i \leq \bar{r} + s\sqrt{n-1}$.
\end{lemma}

\begin{proof}
Using the margin estimates $\hat{\gamma}_1,\dots,\hat{\gamma}_T$, we divide the input space into $K=\lceil R/\hat{\gamma} \rceil-1$ equal regions, so that for any example $x_{t_k}$ in a region $k$, $\hat{y}_{t_k}\ip{x_{t_k}}{u} \geq k\hat{\gamma}$. Let $T_1, \dots,T_K$ be the number of examples in these regions, respectively. Let $\tau_t$ be an indicator variable whose value is 1 if the algorithm makes a prediction mistake on example $x_t$ and 0 otherwise. Let $\varepsilon_k = \sum_{t=i}^{T_k} \tau_{t_i}$ be the number of mistakes made by the algorithm in region $k$, $\varepsilon = \sum_k\varepsilon_k$ be the total number of mistakes made by the algorithm.

We first bound $\|w_{T+1}\|^2$, the weight vector after seeing $T$ examples, from above. If the algorithm makes a mistake at iteration $t$, then $\|w_{t+1}\|^2 = \|w_{t} + \hat{y}_{t}x_{t}\|^2 = \|w_{t}\|^2 + \|x_{t}\|^2 +2\hat{y}_{t}\ip{w_t}{x_{t}} \leq \|w_{t}\|^2 +R^2$, since $\hat{y}_{t}\ip{w_t}{x_{t}}<0$. Since $w_1=0$, we have $\|w_{T+1}\|^2\leq \varepsilon R^2$.

Next, we bound $\ip{w_{{T+1}}}{u}$ from below. Consider the behavior of the algorithm on examples that are located in the farthest region $K$. When a prediction mistake is made in this region at iteration $t_K+1$, we have $\ip{w_{{t_K}+1}}{u}=\ip{w_{t_K}+\hat{y}_{t_K}x_{t_K}}{u}=\ip{w_{t_K}}{u} +\hat{y}_{t_K}\ip{x_{t_K}}{u} \geq \ip{w_{t_K}}{u}+K\hat{\gamma}$.  The weight vector moves closer to $u$ by at least $K\hat{\gamma}$. After the algorithm sees all examples in the farthest region $K$, we have $\ip{w_{T_K+1}}{u} \geq \varepsilon_K K\hat{\gamma}$ (since $w_1=0$), and similarly for region ${K-1}$, $\ip{w_{T_{(K-1)}+1}}{u} \geq \varepsilon_{K} K \hat{\gamma} + \varepsilon_{K-1} (K-1) \hat{\gamma}$, and so on for other regions. Therefore, after the algorithm has seen $T$ examples,  we have 
\begin{align*}
\ip{w_{T+1}}{u}  \geq \sum_{k={1}}^{K} \varepsilon_{k} k \hat{\gamma}
\geq \left(\frac{\varepsilon}{K}-\varepsilon_s\sqrt{K-1}\right)\left(\frac{K(K+1)}{2}\right)\hat{\gamma}\ .
\end{align*}
where we used the Laguerre-Samuelson inequality to lower-bound $\varepsilon_k$ for all $k$, using the mean $\varepsilon/K$ and standard deviation $\varepsilon_s$ of $\{\varepsilon_1,\ldots,\varepsilon_K\}$.

Combining these lower and upper bounds, we get the following quadratic equation in $\sqrt{\varepsilon}$:
\begin{align*}
\left(\frac{\varepsilon}{K}-\varepsilon_s \sqrt{K-1}\right)\left(\frac{K(K+1)}{2}\right)\hat{\gamma}- \sqrt{\varepsilon}R\|u\| \leq 0 \ ,
\end{align*}
whose solution is given by:
\begin{align*}
\sqrt{\varepsilon} \leq \frac{R\|u\| + \sqrt{R^2\|u\|^2+\varepsilon_s K(K+1)^2\sqrt{K-1}\hat{\gamma}^2}}{\hat{\gamma}(K+1)}
\ . 
\end{align*}
\end{proof}

Note that if $\varepsilon_s=0$, i.e., when the number of mistakes made by the perceptron in each of the regions is the same, then we get the following mistake bound: 
\begin{align*}
\varepsilon  \leq \frac{4R^2\|u\|^2}{\hat{\gamma}^2(K+1)^2} \ ,
\end{align*}
clearly improving the mistake bound of the standard perceptron algorithm. However, $\varepsilon_s=0$ is not a realistic assumption. We therefore plot x-fold improvement of the mistake bound as a function of $\varepsilon_s$ for a range of margins $\hat{\gamma}$ in Figure \ref{fig:theory_bound}. The y-axis is the ratio of mistake bounds of interactive perceptron to standard perceptron with all examples scaled to have unit Euclidean length ($R=1$) and $\|u\|=1$. From the figure, it is clear that even when $\varepsilon_s>0$, it is possible to get non-trivial improvements in the mistake bound.
\begin{figure}[!t]
  \centering
    \includegraphics[width=0.9\textwidth]{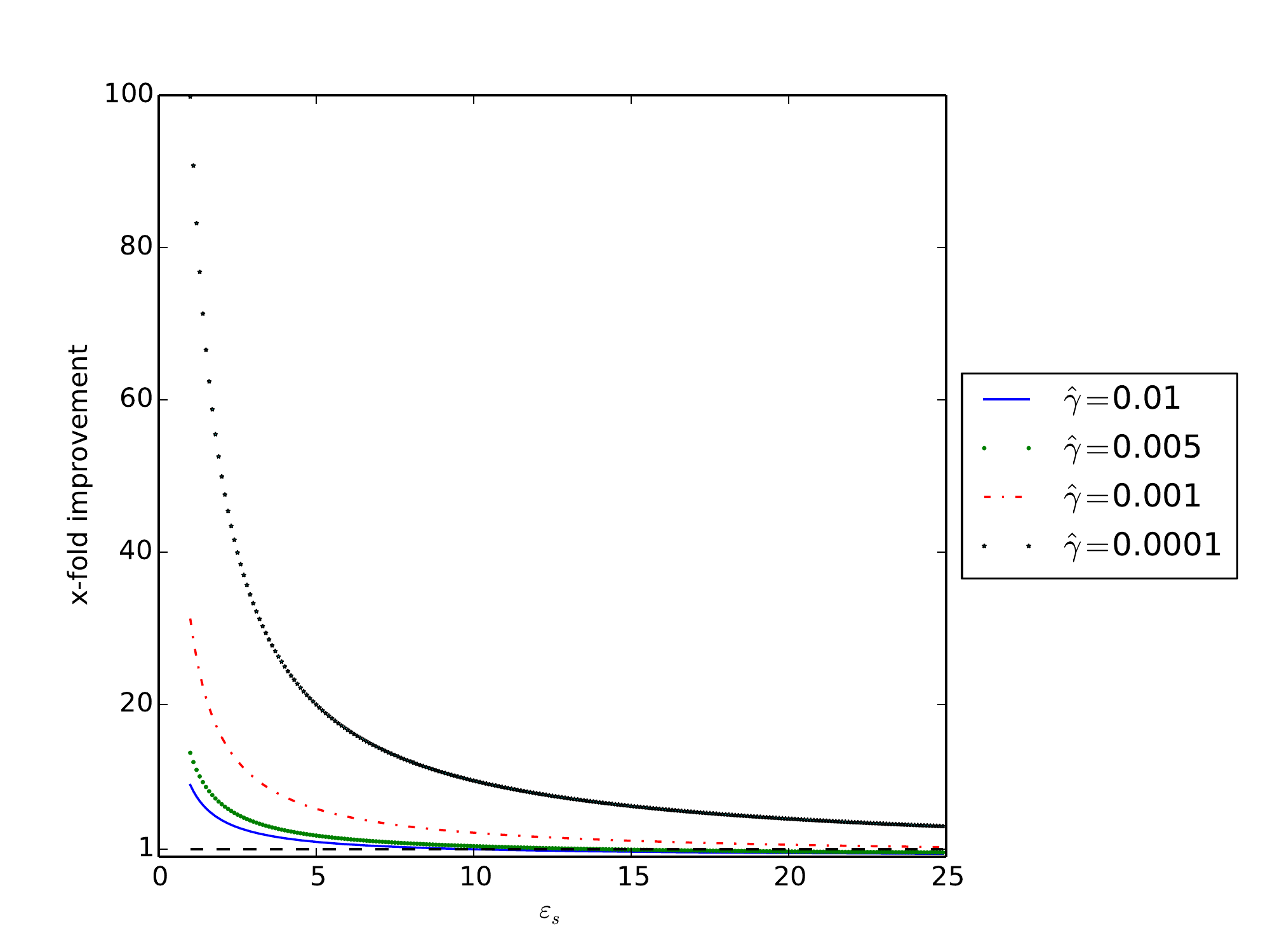}
    \caption[ ]{Illustration of the improvement in the mistake bound of interactive perceptron when compared to standard perceptron. The dashed line is $y=1$.}
    \label{fig:theory_bound}
\end{figure}

The above analysis uses margin and label estimates, $\hat{\gamma}_1,\dots,\hat{\gamma}_T$, $\hat{y}_1,\dots,\hat{y}_T$, from our method described in the Appendix, which may not be exact. We therefore have to generalize the mistake bound to account for noise in these estimates. Let $\{\gamma_1,\dots, \gamma_T\}$ be the true margins of examples. 
Let $\epsilon_{\gamma_u}, \epsilon_{\gamma_{\ell}} \in (0,1]$ denote margin noise factors such that $\hat{\gamma}_t/\epsilon_{\gamma_{\ell}} \geq \gamma_t \geq \epsilon_{\gamma_u} \hat{\gamma}_t$, for all $t \in \{1,\ldots,T\}$. These noise factors will be useful to account for overestimation and underestimation in $\hat{\gamma}_t$, respectively. 

Label noise essentially makes the classification problem linearly inseparable, and so the mistake bound can be analyzed using the method described in the work of Freund and Schapire \cite{FreundS99} (see Theorem 2). Here, we define the deviation of an example $x_t$ as $\delta_t = \max(0, \hat{\gamma}/\epsilon_{\gamma_{\ell}}-\hat{y}_t\ip{u}{x_t})$ and let $\Delta = \sqrt{\sum_t \delta_t^2}$. 
As will become clear in the analysis, if $\hat{\gamma}$ is overestimated, then it does not affect the worst-case analysis of the mistake bound in the presence of label noise.
If the labels were accurate, then $\delta_t=0$, for all $t \in \{1,\ldots,T\}$. With this notation in place, we are ready to analyze the mistake bound of perceptron in the noisy setting. Below, we state and prove the theorem for $\varepsilon_s=0$, i.e., when the number of mistakes made by the perceptron is the same in all the $K$ regions. The analysis is similar for $\varepsilon_s>0$, but involves tedious algebra and so we omit the details in this paper.

\begin{theorem}
\label{main_theorem2}
Let $((x_1,\hat{y}_1,\hat{\gamma}_1), \dots, (x_T,\hat{y}_T,\hat{\gamma}_T))$ be a sequence of training examples along with their label and margin estimates, sorted in descending order based on the margin estimates, and with $\|x_t\| \leq R$ for all $t \in \{1,\ldots,T\}$. Let $\hat{\gamma}=\min (\hat{\gamma}_1,\dots,\hat{\gamma}_T)=\hat{\gamma}_T$ and $~K=\lceil R/\hat{\gamma} \rceil-1$. Suppose there exists a vector $u$ such that $\hat{y}_t\ip{u}{x_t} \geq \hat{\gamma}$ for all the examples. Divide the input space into $K$ equal regions, so that for any example $x_{t_k}$ in a region $k$ it holds that $\hat{y}_{t_k}\ip{x_{t_k}}{u} \geq k\hat{\gamma}$. Assume that the number of mistakes made by the perceptron is equal in all the $K$ regions. Let $\{\gamma_1,\dots, \gamma_T\}$ denote the true margins of the examples, and suppose there exists $\epsilon_{\gamma_u}, \epsilon_{\gamma_{\ell}} \in (0,1]$ such that $\hat{\gamma}_t/\epsilon_{\gamma_{\ell}} \geq \gamma_t \geq \epsilon_{\gamma_u} \hat{\gamma}_t$ for all $t \in \{1,\ldots,T\}$. Define $\delta_t = \max(0, \hat{\gamma}/\epsilon_{\gamma_{\ell}}-\hat{y}_t\ip{u}{x_t})$ and  let $\Delta = \sqrt{\sum_t \delta_t^2}$. 

Then, the total number of mistakes $\varepsilon$ made by the perceptron algorithm on the sequence of training examples is bounded from above via:
\begin{align*}
\varepsilon  \leq \frac{4(\Delta + R\|u\|)^2}{\epsilon_{\gamma_u}^2\hat{\gamma}^2(K+1)^2} \ .
\end{align*}
\end{theorem}
\begin{proof} (Sketch)
Observe that margin noise affects only the analysis that bounds $\ip{w_{{T+1}}}{u}$ from below. When a prediction mistake is made in region $K$, the weight vector moves closer to $u$ by at least $K\epsilon_{\gamma_u}\hat{\gamma}$. After the algorithm sees all examples in the farthest region $K$, we have $\ip{w_{T_K+1}}{u} \geq \varepsilon_K K \epsilon_{\gamma_u} \hat{\gamma}$ (since $w_1=0$). Therefore, margin noise has the effect of down-weighting the bound by a factor of $\epsilon_{\gamma_u}$. The rest of the proof follows using the same analysis as in the proof of Theorem \ref{main_theorem}. Note that margin noise affects the bound only when $\hat{\gamma}_t$ is overestimated because the margin appears only in the denominator when $\varepsilon_s=0$.

To account for label noise, we use the proof technique in Theorem 2 of Freund and Schapire's paper \cite{FreundS99}. The idea is to project the training examples into a higher dimensional space where the data becomes linearly separable and then invoke the mistake bound for the separable case. Specifically, for any example $x_t$, we add $T$ dimensions and form a new vector such that the first $n$ coordinates remain the same as the original input, the $(n+t)$'th coordinate gets a value equal to $C$ (a constant to be specified later), and the remaining coordinates are set to zero. Let $x^{\prime}_t \in \R^{n+T}$ for all $t \in \{1,\ldots,T\}$ denote the examples in the higher dimensional space. Similarly, we add $T$ dimensions to the weight vector $u$ such that the first $n$ coordinates remain the same as the original input, and the $(n+t)$'th coordinate is set to $\hat{y}_t \delta_t/C$, for all $t \in \{1,\ldots,T\}$. Let $u^{\prime} \in \R^{n+T}$ denote the weight vector in the higher dimensional space.

With the above construction, we have $\hat{y}_t \ip{u^{\prime}}{x_t^{\prime}} = \hat{y}_t\ip{u}{x_t} + \delta_t \geq \hat{\gamma}/\epsilon_{\gamma_{\ell}}$. In other words, examples in the higher dimensional space are linearly separable by a margin $\hat{\gamma}/\epsilon_{\gamma_{\ell}}$. Also, note that the predictions made by the perceptron in the original space are the same as those in the higher dimensional space. To invoke Theorem \ref{main_theorem}, we need to bound the length of the training examples in the higher dimensional space, which is $\|x^{\prime}_t\|^2 \leq R^2 + C^2$. Therefore the number of mistakes made by the perceptron is at most  $4(R^2+C^2)(\|u\|^2+\Delta^2/C^2)/(\epsilon_{\gamma_u}^2\hat{\gamma}^2(K+1^2))$. It is easy to verify that the bound is minimized when $C=\sqrt{R\Delta/\|u\|}$, and hence the number of mistakes is bounded from above by $4(\Delta+R\|u\|)^2(\epsilon_{\gamma_u}^2\hat{\gamma}^2(K+1)^2)$.
\end{proof}

\section{Empirical Analysis}
We conducted experiments on synthetic and benchmark datasets.\footnote{Software is available at \url{https://github.com/svembu/ilearn}.} For all datasets, we simulated annotators to generate (noisy) labels in the following way. For a given set of training examples, $\{(x_i,y_i)\}_{i=1}^m$, we first trained a linear model $f(x) = \ip{w}{x}$ with the true binary labels and normalized the scores ${f_i}$ of all examples to lie in the range $[-1,+1]$. We then transformed the scores via $\tilde{f} = 2 \times (1 - (1 / (1 + \exp(p \times -2.5 * |f|))))$, so that examples close to the decision boundary with $f_i\approx0$ get a score $\tilde{f_i}\approx 1$ and those far away from the decision boundary with $f_i\approx \pm 1$ get a score $\tilde{f_i} \approx 0$ as shown in Figure \ref{fig:scores}.
\begin{figure}[!t]
  \centering
    \includegraphics[width=0.9\textwidth]{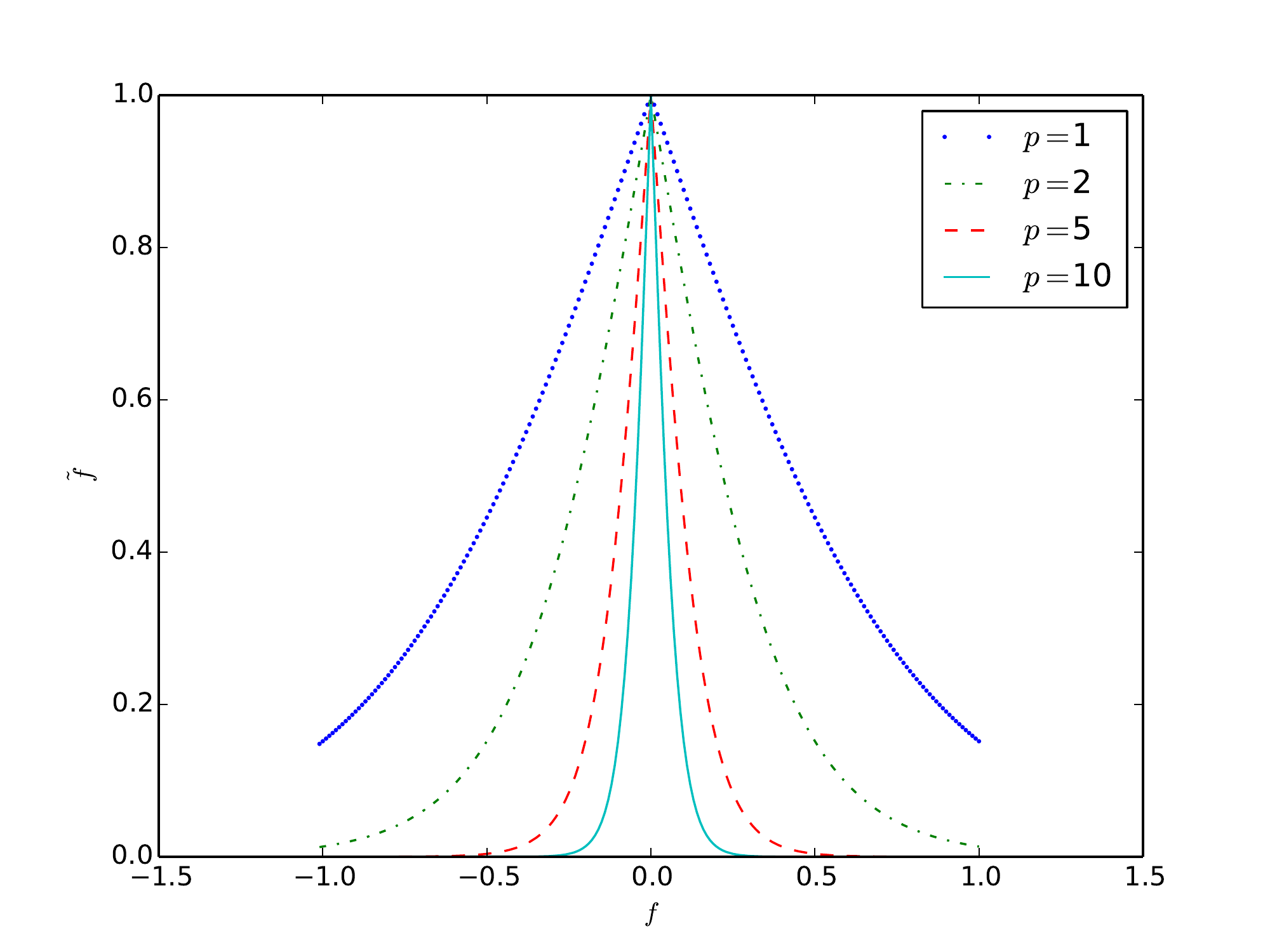}
    \caption[ ]{Illustration of the function used to convert the score from a linear model to a probability that is used to simulate noisy labels from annotators.}
    \label{fig:scores}
\end{figure}

For each example $x_i$, we generated $L$ copies of its true label, and then flipped them based on a Bernoulli trial with success probability $\tilde{f_i}/2$. This has the effect of generating (almost) equal numbers of labels with opposite sign and hence maximum disagreement among labels for examples that are close to the decision boundary. In the other extreme, labels of examples located far away from the decision boundary will not differ much. Furthermore, we flipped the sign of all labels based on a Bernoulli trial with success probability $\tilde{f_i}$ if the majority of labels is equal to the true label. This ensures that the majority of labels are noisy for examples close to the decision boundary. The noise parameter $p$  controls the amount of noise injected into the labels -- high values result in weak disagreement among annotators and low label noise, as shown in Figure \ref{fig:scores}. Table \ref{tbl:sim_labels} shows the noisy labels generated by ten annotators for $p=1$ on a simple set of one-dimensional examples in the range $[-1, +1]$. As is evident from the table, the simulation is designed in such a way that an example close to (resp. far away from) the decision boundary will have a strong (resp. weak) disagreement among its labels. 
\begin{table}[!t]
\caption{Labels provided by a set of 10 simulated annotators for a one-dimensional dataset in the range [-1,+1]}
\begin{center}
  \begin{tabular}{  C{1cm} | C{1cm} | C{1cm} | C{6cm} | C{3cm} }
      $f$ & $\tilde{f}$ & True label & Noisy Labels & Label disagreement, $d$ (Eqn. \eq{eq:disagree}) \\
    \hline
-1.0 & 0.15 & -1 & [-1, -1, -1, -1, -1, -1, -1, -1, -1, -1] & 0 \\
-0.9 & 0.19 & -1 & [ 1, -1, -1, -1, -1, -1, -1, -1, -1, -1] & 36\\
-0.8 &0.24 & -1 & [-1, -1, -1, -1,  1, -1, -1, -1, -1, -1] &36\\
-0.7 &0.3 & -1 & [-1,  1,  1, -1, -1, -1, -1,  1, -1, -1] &84\\
-0.6 &0.36 & -1 &[-1, -1, -1, -1, -1, -1, -1, -1, -1, -1] &0\\
-0.5 &0.45 & -1 & [ 1, -1,  1, -1, -1, -1, -1, -1, -1, -1] &64\\
-0.4 &0.55 & -1 &[ 1,  1,  1,  1,  1,  1, -1,  1,  1, -1] &64\\
-0.3 &0.64 & -1 &[ 1,  1,  1, -1,  1,  1, -1, -1,  1,  1] &84\\
-0.2 &0.76 & -1 &[ 1,  1,  1, -1,  1,  1,  1,  1,  1, -1] &64\\
-0.1 & 0.88 & -1 &[ 1, -1,  1, -1,  1, -1,  1,  1. -1,  1] &96\\
0 &1 & 1 &[-1,  1,  1, -1, -1, -1, -1,  1, -1, -1] &84\\
0.1& 0.88& 1 & [-1,  1, -1, -1, -1,  1,  1, -1,  1,  1] &100\\
0.2& 0.76& 1 &[-1, -1,  1, -1,  1, -1, -1, -1, -1, -1] &64\\
0.3& 0.64& 1 &[-1, -1,  1, -1, -1, -1,  1, -1,  1,  1] &96\\
0.4& 0.54 & 1 &[ 1, -1, -1,  1, -1, -1, -1, -1,  1, -1] &84\\
0.5& 0.45 & 1 &[ 1, -1, -1, -1, -1, -1, -1, -1, -1, -1] &36\\
0.6& 0.36& 1 &[ 1,  1,  1,  1,  1,  1,  1,  1, -1,  1] &36\\
0.7& 0.3& 1 &[ 1,  1,  1,  1,  1, -1, -1,  1, -1,  1] &84\\
0.8& 0.24& 1 &[-1, -1, -1, -1,  1, -1, -1, -1, -1, -1] &36\\
0.9& 0.19& 1 &[ 1,  1,  1,  1,  1,  1,  1,  1,  1,  1] &0\\
1.0& 0.15& 1 &[ 1,  1,  1,  1,  1,  1,  1,  1,  1, -1] &36\\
  \end{tabular}
    \label{tbl:sim_labels}
\end{center}
  \end{table}

\subsection{Synthetic Datasets}
We considered binary classification problems with examples generated from two 10-dimensional Gaussians centered at $\{-0.5\}^{10}$ and $\{+0.5\}^{10}$ with unit variance. We generated noisy labels using the procedure described above. Specifically, we simulated 12 annotators -- one of them always generated the true labels, another flipped all the true labels, the remaining 10 flipped labels using the simulation procedure described above. We randomly generated 100 datasets, each of them having 1000 training examples equally divided between the two classes. We used half of the data set for training and the other half for testing. In each experiment, we tuned the regularization parameter ($\lambda$ in Equation \ref{eq:opt_interactive}) by searching over the range $\{2^{-14},2^{-12},\ldots,2^{12},2^{14}\}$ using 10-fold cross validation on the training set, retrained the model on the entire training set with the best-performing parameter, and report the performance of this model on the test set. We experimented with a range of $(\alpha, p)$ values. Recall that the parameter $\alpha$ influences the reweighting of examples with small values placing (almost) equal weights on all the examples and large values placing a lot of weight on examples whose labels have a large disagreement (Figure \ref{fig:reweight_examples}). The parameter $p$ as mentioned before controls label noise. We compared the performance of the algorithm in interactive and non-interactive modes described in Section \ref{sec:crowdsourcing_method}. The non-interactive algorithm is the one described in Raykar \emph{et al.}'s paper \cite{RaykarYZVFBM10}.

The results are shown in Table \ref{tbl:sim_results}. We use area under the receiver operating characteristic curve (AU-ROC) and area under the precision-recall curve (AU-PRC) as performance metrics.
\begin{table}[!t]
  \caption{Experimental results on synthetic datasets. Also shown in the table are two-sided p-values of the Wilcoxon signed-rank test.}
\begin{center}
  \begin{tabular}{  l | C{2cm} | c || C{2cm} | c}
    Parameters & AU-ROC (\#wins) & p-value & AU-PRC (\#wins) & p-value \\
    \hline
    $\alpha=0.1, p=1$ & 61 & 0.0542 & 61 & 0.0535 \\ 
    $\alpha=1, p=1$ & 75 & $4.11 \times 10^{-11}$ & 75 & $2.97 \times 10^{-11}$ \\ 
    $\alpha=2, p=1$ & 88 & $6.26 \times 10^{-14}$ & 89 & $3.65 \times 10^{-14}$\\ 
    $\alpha=5, p=1$ &  73 & $3.56 \times 10^{-5}$ & 75 & $2.26 \times 10^{-5}$ \\ 
    \hline
    $\alpha=0.1, p=2$ & 50 & 0.5684 & 49 & 0.6199 \\ 
    $\alpha=1, p=2$ & 51 & 0.1822 & 50 & 0.1799  \\ 
    $\alpha=2, p=2$ & 61 & 0.0007 & 61 & 0.0009 \\ 
    $\alpha=5, p=2$ & 49 & 0.5075 & 49 & 0.7463 \\ 
    \hline
    $\alpha=0.1, p=5$ & 32 & 0.4784 & 34 & 0.8479 \\ 
    $\alpha=1, p=5$ & 44 & 0.0615 & 42 & 0.0817  \\ 
    $\alpha=2, p=5$ & 48 & 0.2334 & 49 & 0.3661 \\ 
    $\alpha=5, p=5$ & 37 & 0.0562 & 33 & 0.028 \\ 
  \end{tabular}
    \label{tbl:sim_results}
\end{center}
  \end{table}
In the table, we show the number of times the AU-ROC and the AU-PRC of the interactive algorithm is higher than its non-interactive counterpart (\#wins out of 100 datasets). We also show the two-sided p-value from the Wilcoxon signed-rank test. From the results, we note that the performance of the interactive algorithm is not significantly better than its non-interactive counterpart for small and large values of $\alpha$. This is expected because small values of $\alpha$ reweight examples (almost) uniformly and so there is not much to gain when compared to running the algorithm in the non-interactive mode. In the other extreme, large values of $\alpha$ tend to discard a large number of examples close to the decision boundary thereby degrading the overall performance of the algorithm in the interactive mode. $\alpha=2$ gives the best performance. We also note that for high values of $p$, i.e., weak disagreement among annotators and hence low label noise, the interactive algorithm offers no statistically significant gains when compared to the non-interactive algorithm. This, again, is as expected.

\subsection{Benchmark Datasets}
We used LibSVM benchmark\footnote{\url{https://www.csie.ntu.edu.tw/~cjlin/libsvmtools/datasets/}} datasets in our experiments. We selected binary classification datasets with at most 10,000 training examples and 300 features (Table \ref{tbl:datasets}), so that we could afford to train multiple linear models (100 in our experiments) for every dataset using standard solvers and also afford to tune hyperparameters carefully in a reasonable amount of time.
\begin{table}[!t]
 \caption{Datasets used in the experiments}
\begin{center}
  \begin{tabular}{  l | C{2cm} | C{2cm} |  C{2cm}}
    Name &  No. of training examples & No. of test examples &  No. of features \\
    \hline
    a1a & 1,605 & 30,956 & 123 \\
    a2a & 2,265	& 30,296 & 123 \\
    a3a & 3,185	& 29,376 & 123 \\
    a4a & 4,781	& 27,780 & 123 \\
    a5a & 6,414	& 26,147 & 123 \\
    australian &690 & -  & 14 \\
    breast-cancer & 683 & - & 10 \\
    diabetes & 768 & - & 8 \\
    fourclass & 862 & - & 8 \\
    german.nuner & 1000 & - & 24 \\
    heart & 270 & - & 13 \\
    ionosphere & 351 & - & 34 \\
    liver-disorders & 345 & - & 6 \\
    splice & 1,000 &	2,175 &	60\\
    sonar & 208	& - &	60\\
    w1a & 2,477 &	47,272 &	300 \\
    w2a & 3,470	& 46,279 &	300 \\
    w3a & 4,912	& 44,837 &	300 \\
    w4a & 7,366	& 42,383 &	300 \\
    w5a & 9,888	& 39,861 &	300 \\
  \end{tabular}
   \label{tbl:datasets}
\end{center}
  \end{table}
We generated noisy labels with the same procedure used in our experiments on synthetic data. Also, we tuned the regularization parameter in an identical manner. For datasets with no predefined training and test splits, we randomly selected 75\% of the examples for training and used the rest for testing. For each dataset, we randomly generated 100 sets of noisy labels from the 12 annotators resulting in 100 different random versions of the dataset. The results are shown in Table \ref{tbl:real_results}. 
\begin{table}[!t]
 \caption{Experimental results on benchmark datasets. Statistically \emph{insignificant} results ($\text{p-value}>0.01$) are indicated with an asterisk ($\ast$). 
 }
\begin{center}
  \begin{tabular}{  l | C{2.75cm} | C{2.75cm} | C{2.75cm}}
    Dataset &  $\alpha=1, p=1$ & $\alpha=2, p=1$ &  $\alpha=5, p=1$ \\
    & \scriptsize{AU-ROC} $\mid$ \scriptsize{AU-PRC} & \scriptsize{AU-ROC} $\mid$ \scriptsize{AU-PRC} & \scriptsize{AU-ROC} $\mid$ \scriptsize{AU-PRC} \\
    \hline
    a1a & 59 $\mid$ 65 & 79 $\mid$ 73 & 66 $\mid$ 53* \\
    a2a & 64 $\mid$ 68 & 83 $\mid$ 79 & 66 $\mid$ 57  \\
    a3a & 74 $\mid$ 76 & 88 $\mid$ 83 &  71 $\mid$ 63 \\
    a4a & 80 $\mid$ 83 & 88 $\mid$ 84 & 67 $\mid$ 70 \\
    a5a &  82 $\mid$ 83 & 95 $\mid$ 92 & 79 $\mid$ 74 \\
    australian & 44* $\mid$ 43* & 47* $\mid$ 50* & 41* $\mid$ 43* \\
    breast-cancer & 51 $\mid$ 54 & 60 $\mid$ 63 & 61 $\mid$ 62 \\
    diabetes &  84 $\mid$ 80 & 81 $\mid$ 76 & 67 $\mid$ 65 \\
    fourclass & 38* $\mid$ 37* & 41* $\mid$ 39* & 46* $\mid$ 41 \\
    german.numer & 79 $\mid$ 75 & 73 $\mid$ 67 & 49* $\mid$ 48* \\
    heart & 55* $\mid$ 52* & 63* $\mid$ 58* & 57* $\mid$ 50* \\
    ionosphere & 56* $\mid$ 56* & 64 $\mid$ 65 & 49* $\mid$ 54* \\
    liver-disorders & 67 $\mid$ 64 & 61 $\mid$ 60* & 52* $\mid$ 54* \\
    splice & 93 $\mid$ 93 & 90 $\mid$ 90 & 70 $\mid$ 68 \\
    sonar & 62 $\mid$ 66 & 66 $\mid$ 64 & 56* $\mid$ 50* \\
    w1a & 38* $\mid$ 37* & 48* $\mid$ 46* & 28 $\mid$ 32 \\
    w2a & 64 $\mid$ 57 & 69 $\mid$ 61 & 46* $\mid$ 44* \\
    w3a & 54* $\mid$ 48* & 52* $\mid$ 48* & 34 $\mid$ 33 \\
    w4a & 85 $\mid$ 81 & 79 $\mid$ 74 & 71 $\mid$ 59* \\
    w5a & 89 $\mid$ 80 & 89 $\mid$ 78 & 75 $\mid$ 66 \\
  \end{tabular}
    \label{tbl:real_results}
\end{center}
  \end{table}
In the table, we again show the number of times the AU-ROC and the AU-PRC of the interactive algorithm is higher than its non-interactive counterpart (\#wins out of 100 datasets). We report results on only a subset of $(\alpha, p)$ values that were found to give good results based on our experimental analysis with synthetic data. From the table, it is clear that the interactive algorithm performs significantly better than its non-interactive counterpart on the majority of datasets. On datasets where its performance was worse than that of the non-interactive algorithm, the results were not statistically significant across all parameter settings. 

As a final remark, we would like to point out that the performance of the interactive algorithm dropped on some of the datasets with class imbalance. We therefore subsampled the training sets (using a different random subset in each of the 100 experiments for the given dataset) to make the classes balanced. We believe the issue of class imabalance is orthogonal to the problem we are addressing, but needs further investigation and so we leave this open for future work.

\section{Concluding Remarks}
Our experiments clearly demonstrate the benefits of interactive learning and how disagreement among annotators can be utilized to improve the performance of supervised learning algorithms. Furthermore, we presented theoretical evidence by analyzing the mistake bound of perceptron. The question as to whether annotators in real world scenarios behave according to our simulation model, i.e., if they tend to disagree more on difficult examples located close to the decision boundary when compared to easy examples farther away, is an open one. However, if this assumption holds then our experiments and theoretical analysis show that learning can be improved.

In real-world crowdsourcing applications, an example is typically labeled only by a subset of annotators. Although we did not consider this setting, we believe we could still use the disagreement among annotators to reweight examples, but the algorithm would require some modifications to handle missing labels. We leave this setting open for future work.

\bibliographystyle{unsrt}

\bibliography{bib}

\section*{Appendix: Estimating the Margin of an Example}
\label{sec:margin_estimates}
The margin of an example $x$ with respect to a linear function $f(\cdot)$ is defined as $\gamma(x; f) = |f(x)| = |\ip{w}{x}|$. Examples close to the decision boundary will have a small margin and those that are farther away will have a large margin. We assume that an annotator labels an example $x$ using the true labels of all neighboring examples in a ball of some radius centered at $x$. The size of an annotator's ball is inversely proportional to her strength (expertise). This model of annotators is similar to the one used in Urner \emph{et al.}'s analysis \cite{UrnerBS12}. Note that neither the true labels nor the size of the annotator's ball is known to us. Our only input is a set of $m$ examples with corresponding (noisy) labels from $L$ annotators. Given this input, the goal is to estimate the radius of the annotator's ball. This will then allow us to estimate the margin of an example, i.e., its distance from the separating hyperplane. We proceed in two steps: first, we describe a method to estimate the annotators' expertise scores $\{z_1,z_2,\ldots,z_L\}$ and the ground-truth labels; second, we use these estimates to compute the radii of the annotators' balls. 

\subsubsection*{Estimating an annotator's expertise, $z$.}
We use a variant of kernel target alignment \cite{CristianiniSEK01} to estimate the expertise score of each annotator. Let $K$ denote the (centered) kernel matrix on the input examples, i.e., $K \in [-1,+1]^{m \times m}$ with $k_{ij} = \ip{\phi(x_i)}{\phi(x_j)}$, where $\phi(\cdot)$ is a feature map. For linear models, the entries of the kernel matrix are pairwise dot-products of training examples. We consider the following optimization problem to estimate an annotator's expertise score:
\begin{equation*}
\label{eq:opt_z}
\begin{aligned}
\hat{z} = \argmin_{z \in [0,1]^L} \sum_{i=1}^m\sum_{j=1}^m \left(k_{ij} - \frac{1}{L}\sum_\ell z_\ell y_i^{(\ell)} y_j^{(\ell)}\right)^2 \ .
\end{aligned}
\end{equation*}
This is a constrained least-squares regression problem. The complexity of this optimization problem is quadratic in the number of examples. However, we can use stochastic (projected) gradient descent to remove the dependence on the number of examples. 

The ground-truth label of an example $x_i$ can be estimated by taking the weighted average of labels provided by the annotators, i.e., for each given tuple  $\left(x_i,y_i^{(1)},y_i^{(2)},\ldots,y_i^{(L)}\right)$, we form a new training example $(x_i,\hat{y}_i)$ with $\hat{y}_i = \sgn\left(\sum_{\ell}z_{\ell}y_i^{(\ell)}\right)$ and let $\hat{S} = \{(x_i,\hat{y}_i),\ldots,(x_i,\hat{y}_m)\}$.

\subsubsection*{Estimating the radius of an annotator's ball.}
Let $r_1,r_2,\ldots,r_L \geq 0$ denote the radii of the annotators' balls. Let $B_r(x) = \{z \in \Xcal \mid d(x,z)\leq r\}$ denote the ball of radius $r$ centered at $x$, with $d(\cdot,\cdot)$ being a distance metric, such as the Euclidean distance for linear models, defined on the input space $\Xcal$. Given the expertise score $z_{\ell}$ for an annotator $\ell$, we estimate the radius $r_{\ell}$ of her ball by solving the following univariate optimization problem:
\begin{equation*}
\label{eq:radii}
\begin{aligned}
\hat{r}_{\ell} = \argmin_{r \in \R_+} \sum\limits_{i=1}^m \left( \frac{\sum\limits_{(z,\hat{y}) \in B_r(x_i) \cap \hat{S}} \hat{y}}{|B_r(x_i) \cap \hat{S}|} - y_i^{(\ell)}\right)^2 \ .
\end{aligned}
\end{equation*}
Intuitively, the above optimization problem is trying to estimate the radius of the annotator's ball by minimizing the squared difference between the (noisy) label of the annotator and the average of the estimates of true labels of all neighboring examples in the ball.

\subsubsection*{Putting it all together.}
Given a training example $x$, its noisy labels $(y^{(1)},\ldots,$ $y^{(L)})$, an estimate of the ground-truth label $\hat{y}$, and the radius estimates of the annotators' balls, we compute a lower bound on the margin of $x$, i.e., its distance from the decision boundary, as follows. Centered at $x$, we draw nested balls of increasing size, one for each annotator using her radius. Starting from the annotator with the smallest ball, we compare her noisy label with the ground-truth label estimate. At some ball/expert, the noisy label and the ground-truth label estimate will differ, and the radius of this ball is a lower bound on the distance of $x$ from the decision boundary.

\end{document}